\documentclass{article}
\usepackage[final]{nips_2017}
\usepackage[utf8]{inputenc} 
\usepackage[T1]{fontenc}    
\usepackage{hyperref}       
\usepackage{url}            
\usepackage{booktabs}       
\usepackage{amsfonts}       
\usepackage{nicefrac}       
\usepackage{microtype}      

\usepackage{amsopn}
\usepackage{amsthm}
\usepackage{etex}
\usepackage{graphicx}
\usepackage{endnotes}
\usepackage{hyperref}
\usepackage{epsfig,psfrag}
\usepackage{pst-all}
\usepackage{amssymb,amsfonts,upref,cite,epsf,color,bm}
\usepackage{graphicx}
\usepackage{color}
\usepackage{amsmath}
\usepackage{graphicx}
\usepackage{calc}
\usepackage{booktabs}
\usepackage{tikz}
\usepackage{pgfplots}
\newcommand\defeq{:=}

\usepackage{algorithm}
\usepackage{algpseudocode}
\floatname{algorithm}{Algorithm}
\algnewcommand\algorithmicinput{\textbf{Input:}}
\algnewcommand\INPUT{\item[\algorithmicinput]}
\algnewcommand\algorithmicoutput{\textbf{Output:}}
\algnewcommand\OUTPUT{\item[\algorithmicoutput]}

\DeclareMathOperator*{\argmin}{arg\;min}

\newcommand\vect[1]{\mathbf #1}
\newcommand{\va}{\vect{a}}

\newcommand{\vd}{\vect{d}}

\newcommand{\vx}{\vect{x}}  
\newcommand{\vy}{\vect{y}}  
\newcommand{\vz}{\vect{z}}

\newcommand{\mC}{\mathbf{C}}

\newcommand{\mW}{\mathbf{W}}
\newcommand{\measlen}{M}

\newcommand{\signalsize}{N}
\newcommand{\noise}{\bm{\varepsilon}}

\newcommand{\edges}{\mathcal{E}}
\newcommand{\cluster}{\mathcal{C}}
\newcommand{\nodes}{\mathcal{V}}
\newcommand{\graph}{\mathcal{G}}
\newcommand{\samplingset}{\mathcal{M}}

\newcommand{\sigdim}{p}
\newcommand{\flow}{h}
\newcommand{\partition}{\mathcal{F}}

\newcommand{\compbound}{\overline{\partial \partition}}

\newtheorem{theorem}{Theorem}
\newtheorem{definition}[theorem]{Definition}
\newtheorem{lemma}[theorem]{Lemma}

\newcommand{\edge}[2]{\{#1,#2\}}
\usepackage{setspace}

\title{When is Network Lasso Accurate: The Vector Case}

\author{ Nguyen Tran, Saeed Basirian, 
  Alexander Jung  \\
  Department of Computer Science\\
  Aalto University\\
  \texttt{firstname.lastname(at)aalto.fi} \\
}

\begin{document}
	\maketitle
\begin{abstract}
A recently proposed learning algorithm for massive network-structured data sets (big data over networks) is the network 
Lasso (nLasso), which extends the well-known Lasso estimator from sparse models to network-structured datasets. 
Efficient implementations of the nLasso have been presented using modern convex optimization methods. 
In this paper we provide sufficient conditions on the network structure and available label information such that nLasso accurately 
learns a vector-valued graph signal (representing label information) from the information provided 
by the labels of a few data points.  
\end{abstract}

\section{Introduction}
\label{sec_intro}

We consider datasets which are represented by a ``data graph''. The nodes of the data graph represent 
individual data points (e.g., one image out of a image collection) which are are connected by edges according to some 
notion of similarity. This similarity might be induced naturally by the application at hand (e.g., in social networks) or obtained 
from statistical models (probabilistic graphical models) \citep{gmsIcassp2017,gms2017}. 
Beside graph structure, the datasets typically carry label information which we represent by a graph signal \citep{SandrMoura2014}. 

The acquisition of graph signal values (labels) is often expensive, and therefore we are interested in methods for learning the 
entire graph signal from a (small) subset of nodes (sampling set), which is a crucial task in many machine learning problems. 
The learning of the graph signals from a small number of signal samples, which are obtained by manually labelling 
few data points, is enabled by exploiting the tendency of natural graph signals to be smooth. 
More precisely, the smoothness hypothesis, which underlies most (semi-) supervised machine learning methods 
\citep{BishopBook,SemiSupervisedBook}, requires the graph signal to be nearly constant over well connected subset of nodes (clusters). 


In this paper, by extending the program initiated in \citep{WhenIsNLASSO} for scalar graph signals, 
we present sufficient conditions on the network topology and available label information such that the 
nLasso can recover an underlying vector-valued graph signal. In particular, we extend the network 
compatibility condition (NCC) introduced in \citep{WhenIsNLASSO} to vector-valued graph signals. 
The NCC ensures accurate recovery of a smooth vector-valued graph signal from only few signal values (initial labels) 
using nLasso. We then relate the NCC to the existence of certain network flows.




\section{Problem Formulation}
\label{sec_setup}

We consider a graph signal defined over an undirected graph $\graph\!=\!(\nodes,\edges)$, with nodes $\nodes$ 
representing individual data points and undirected edges $\edges$ encoding domain-specific notions of similarity between data points. 
The strength of the connections $\edge{i}{j} \in \edges$ is quantified by non-negative edge weights $W_{ij}$,  
which we collect into a weighted adjacency matrix $\mW \in \mathbb{R}_{+}^{\signalsize \times \signalsize}$ (which is also known as the graph shift matrix \citep{Chen2015}). 

In addition to the graph structure $\graph$, datasets typically convey additional information, e.g., features, labels or 
model parameters associated with individual data points $i \in \nodes$. We represent this additional 
information as a graph signal $\vx[\cdot]: \nodes \rightarrow \mathbb{R}^{\sigdim}$, which maps the node $i\!\in\!\nodes$ 
to the signal vector $\vx[i] \!=\! (x_1[i], \ldots , x_{\sigdim}[i])^T \!\in\! \mathbb{R}^{\sigdim}$. The graph signal vector $\vx[i]$ might represent, 
e.g., the weight vector for a local pricing model in a house prize prediction application (cf.\ \citep{NetworkLasso}). 

The graph signals $\vx[\cdot]$ encountered in many application domains are smooth, i.e., have small \emph{total variation} (TV) 
\begin{equation*} 
\label{equ_def_TV}
\| \vx[\cdot] \|_{\edges} \defeq \sum_{\{i,j\} \in \edges} W_{i,j}  \| \vx[j]\!-\!\vx[i]\|_2 . 
\end{equation*}

Our analysis employs a simple model for smooth graph signals, i.e., clustered graph signals
\begin{equation}
\label{equ_def_clustered_signal_model}
 \vx[i] \!=\! \sum_{\cluster \in \partition} \va_{\cluster} \mathcal{I}_{\cluster}[i],
\end{equation} 
with vectors $\va_{\cluster} \in \mathbb{R}^{\sigdim}$ and the indicator signal $\mathcal{I}_{\cluster}[i] \in \{0,1\}$ 
for subset $\cluster \in \nodes$, i.e., $\mathcal{I}_{\cluster}[i] = 1$ if and only if $i \in \cluster$. The model \eqref{equ_def_clustered_signal_model} 
involves a partition $\partition = \{\cluster_{1}, \ldots, \cluster_{|\partition|}\}$ of $\nodes$ into disjoint subsets $\cluster_{l}$.

While our analysis formally applies to any partition $\partition = \{ \cluster_{1},\ldots,\cluster_{|\mathcal{\partition}|} \}$ used in the model \eqref{equ_def_clustered_signal_model}, 
our results are most useful if the partition conforms with the ``intrinsic (cluster) structure'' of the data graph $\graph$. 
In particular, consider a partition $\partition$ such that the total weight of the cluster boundaries 
\begin{equation*}
\partial \partition = \{ \edge{i}{j} \in \edges: i \in \cluster, j \in \cluster' (\neq \cluster) \}
\end{equation*} 
is small compared to the total weight of intra-cluster edges $\overline{\partial \partition}= \edges \setminus \partial \partition$, i.e., 
$\sum_{\edge{i}{j} \in \partial \partition } W_{i,j} \ll \sum_{\edge{i}{j} \in \overline{\partial \partition}}W_{i,j}$. As can verified easily, for such a 
partition, any signal of the form \eqref{equ_def_clustered_signal_model} is then smooth in the sense of having small TV $\| \vx[\cdot]\|_{\edges}$.




The clustered graph signal is different from the model of band-limited graph signals which is championed in graph signal processing \citep{ChenVarma2015}. 
Indeed, while band-limited graph signals have sparse graph Fourier transform (GFT) coefficients, clustered graph signals 
\eqref{equ_def_clustered_signal_model} have dense GFT coefficients which are spread out over the entire frequency range \citep{WhenIsNLASSO}.


We assume to have access to the graph signal values $\vx[i]$ only for (small) subset of nodes, i.e., the sampling set 
$\samplingset \defeq \{i_{1},\ldots,i_{\measlen}\} \subseteq \nodes$ (typically $|\samplingset| \ll |\nodes|$). 
In particular, we observe 
\begin{equation}
\label{equ_model_initial_labels}
\vy[i] = \vx[i] + \noise[i] \mbox{ for a sampled node } i \in \samplingset. 
\end{equation} 
The error component $\noise[i]$ in \eqref{equ_model_initial_labels} covery any data curation or labelling errors. 

In order to be able to learn the entire  graph signal $\vx[\cdot]$ from partial noisy measurements $\{ \vy[i] \}_{i\!\in\!\samplingset}$ 
we exploit that the true graph signal is smooth, i.e., have small TV $\| \tilde{\vx}[\cdot] \|_{\edges}$. 
Moreover, any reasonable learning algorithm should deliver a graph signal with a small empirical error 
\begin{equation}
\label{equ_def_emp_error}
\widehat{E}(\hat{\vx}[\cdot]) \defeq \sum_{i \in \samplingset} \| \hat{\vx}[i] - \vy[i]\|_1.
\end{equation}
Note that we use the $\ell_1$-norm for the empirical error $\widehat{E}(\hat{\vx}[\cdot])$ \eqref{equ_def_emp_error}, 
which is different from the original lasso, where the squared $\ell_2$-norm was used \citep{BuhlGeerBook}. 

A straightforward recovery method aiming to a small TV $\| \hat{\vx}[\cdot] \|_{\edges}$ and small empirical error $\widehat{E}(\hat{\vx}[\cdot])$ can be formulated as a regularized optimization problem
\begin{align} 
\hat{\vx}[\cdot] & \in \argmin_{\tilde{\vx}[\cdot]:\tilde{\vx}[i] \in \mathbb{R}^{\sigdim}} \widehat{E}(\tilde{\vx}[\cdot])  + \lambda \| \tilde{\vx}[\cdot] \|_{\edges}.  \label{equ_semi_sup_learning_problem}
\end{align}
The tuning parameter $\lambda$ in \eqref{equ_semi_sup_learning_problem} trades off a small empirical error 
$\widehat{E}(\hat{\vx}[\cdot])$ against signal smoothness $\| \hat{\vx}[\cdot] \|_{\edges}$ of the learned signal $\hat{\vx}[\cdot]$. 
A small value of  $\lambda$ enforces the solutions of \eqref{equ_semi_sup_learning_problem} to obtain a s
mall empirical error, whereas, a large value of $\lambda$ enforces the solutions of \eqref{equ_semi_sup_learning_problem} to obtain a small TV, i.e. to be smooth.
The recovery problem \eqref{equ_semi_sup_learning_problem} is a convex problem and can be approached by modern convex optimization 
methods \citep{JungSpawc2016,HannakAsilomar2016,JungHero2016}. 

\vspace*{-2mm}
\section{When is Network Lasso Accurate?} 
\vspace*{-2mm}
\label{sec_main_results} 

We now introduce the network compatibility condition (NCC), which generalizes the 
compatibility conditions for Lasso type estimators \citep{BuhlGeerBook} of ordinary sparse signals. 
Our main contribution is to show that the NCC guarantees any solutions of 
\eqref{equ_semi_sup_learning_problem} allows to accurately learn the true underlying graph signal. 

\begin{definition} 
\label{def_NNSP}
Consider a data graph $\graph = (\nodes, \edges)$ with a particular partition $\partition$ of its nodes $\nodes$. 
A sampling set $\samplingset \subseteq \nodes$ is said to satisfy NCC with constants $K,L>0$, if 
\vspace*{-1mm}
\begin{equation}
\label{equ_ineq_multcompcondition_condition}
K \sum_{i \in \samplingset} \| \vz[i] \|_2 +  \| \vz[\cdot] \|_{\compbound} \geq (L/\sqrt{\sigdim})  \| \vz[\cdot] \|_{\partial \partition} 
\vspace*{-1mm}
\end{equation} 
for any  graph signal $\vz[\cdot]$. 
\end{definition} 

It turns out that, if the sampling set satisfies the NCC, any solution of \eqref{equ_semi_sup_learning_problem} provides an 
accurate estimate of the true underlying graph signal \eqref{equ_def_clustered_signal_model}. 
\begin{theorem} 
\label{lem_NSP1}
Consider a data set represented by data graph $\graph$ and a graph signal $\vx[\cdot]$ of the form \eqref{equ_def_clustered_signal_model}. 
If the sampling set $\samplingset$ satisfies NCC with parameters $L > \sqrt{\sigdim}$ and $K > 0$, then any solution $\hat{\vx}[\cdot]$ of \eqref{equ_semi_sup_learning_problem} with $\lambda \defeq 1/K$ satisfies  
\begin{equation}
 \| \hat{\vx}[\cdot]-\vx[\cdot] \|_{\edges} \!\leq\! K (1\!+\!4\sqrt{\sigdim}/(L\!-\!\sqrt{\sigdim}))  \sum_{i \in \samplingset} \|\noise[i]\|_1 .
 \label{equ_bound_error_TV}
\end{equation} 
\end{theorem}
\begin{proof}
Consider an arbitrary solution $\hat{\vx}[\cdot]$ of \eqref{equ_semi_sup_learning_problem} and denote the difference 
between $\hat{\vx}[\cdot]$ and the true underlying clustered signal $\vx[\cdot]$ as $\tilde{\vx}[\cdot] \defeq \hat{\vx}[\cdot] - \vx[\cdot]$. 
By \eqref{equ_semi_sup_learning_problem},
\begin{equation}
\label{equ_inequ_basic_1}
\hspace{-1mm}\sum_{i \in \samplingset} \|\hat{\vx}[i]- \vy[i]\|_1 \!+\! \lambda \| \hat{\vx}[\cdot] \|_{\edges}  \!\leq\! \sum_{i \in \samplingset} \|\noise[i]\|_1 \!+\!  \lambda \| \vx[\cdot] \|_{\edges}.\hspace{-1mm}
\end{equation} 
Since the true graph signal $\vx[\cdot]$ satisfies \eqref{equ_def_clustered_signal_model}, we have 
$\| \vx[\cdot] \|_{\compbound}=0$ and $\| \tilde{\vx}[\cdot] \|_{\compbound}=\| \hat{\vx}[\cdot] \|_{\compbound}$. 
Combining the decomposition property and triangle inequality for the semi-norm $\| \cdot \|_{\edges}$ with \eqref{equ_inequ_basic_1}, 
\begin{align}
\sum_{i \in \samplingset} \|\hat{\vx}[i] \!-\! \vy[i]\|_1 + \lambda \| \hat{\vx}[\cdot] \|_{\compbound}  &\leq  \sum_{i \in \samplingset} \|\noise[i]\|_1 + \lambda \| {\vx}[\cdot] \|_{\partial \partition}  - \lambda \| \hat{\vx}[\cdot] \|_{\partial \partition} \nonumber\\
\Rightarrow \sum_{i \in \samplingset} \|\hat{\vx}[i] \!-\! \vy[i]\|_1 + \lambda \| \tilde{\vx}[\cdot] \|_{\compbound}  &\leq \sum_{i \in \samplingset} \|\noise[i]\|_1 + \lambda \| \tilde{\vx}[\cdot] \|_{\partial \partition}.
\label{equ_inequ_basic_2}
\end{align}
By triangle inequality, 
\begin{align}
\sum_{i \in \samplingset} \|\hat{\vx}[i] \!-\! \vy[i]\|_1 \stackrel{\eqref{equ_model_initial_labels}}{=} \sum_{i \in \samplingset} \|\hat{\vx}[i] \!-\! \vx[i] - \noise[i] \|_1
\geq  \sum_{i \in \samplingset} \|\tilde{\vx}[i]\|_2 \!-\! \sum_{i \in \samplingset} \|\noise[i]\|_1, \nonumber
\end{align}
where we have used $\|\tilde{\vx}[i]\|_1 \geq \|\tilde{\vx}[i]\|_2$. Therefore,
\begin{align}
\label{equ_inequ_basic_sampleset}
\hspace{-2mm}\max\{0, \sum_{i \in \samplingset} \|\tilde{\vx}[i]\|_2 \!-\! \sum_{i \in \samplingset} \|\noise[i]\|_1\} \!\leq\! \sum_{i \in \samplingset} \|\hat{\vx}[i] \!-\! \vy[i]\|_1.
\end{align}
Applying \eqref{equ_inequ_basic_sampleset} into \eqref{equ_inequ_basic_2} yields 
\begin{align}
\label{equ_upper_bound_complement_partition}
\lambda \| \tilde{\vx}[\cdot] \|_{\compbound}  &\leq \sum_{i \in \samplingset} \|\noise[i]\|_1 + \lambda \| \tilde{\vx}[\cdot] \|_{\partial \partition},
\end{align}
and
\begin{align}
\hspace{-3mm} \sum_{i \in \samplingset} \|\tilde{\vx}[i]\|_2 \!+\! \lambda \| \tilde{\vx}[\cdot] \|_{\compbound}  \leq 2 \sum_{i \in \samplingset} \|\noise[i]\|_1 \!+\! \lambda \| \tilde{\vx}[\cdot] \|_{\partial \partition}.\hspace{-1mm}
\label{equ_inequ_basic_3}
\end{align}
Since we assume NNC holds for $\samplingset$, inequality \eqref{equ_ineq_multcompcondition_condition} applies to $\tilde{\vx}[\cdot]$, i.e.,
\begin{equation} 
\label{equ_inequ_diff_signal}
(1/K) (L/\sqrt{\sigdim}) \|\tilde{\vx}[\cdot]  \|_{\partial \partition} \leq \sum_{i \in \samplingset} \|\tilde{\vx}[i] \|_2 +  (1/K) \| \tilde{\vx}[\cdot] \|_{\compbound}. 
\end{equation} 
Inserting \eqref{equ_inequ_diff_signal} into \eqref{equ_inequ_basic_3} and using $\lambda \defeq 1/K$, yields 
\begin{equation}
\label{equ_upper_bound_partial_partition}
\lambda(L/\sqrt{\sigdim}-1)\| \tilde{\vx}[\cdot] \|_{\partial \partition} \leq 2  \sum_{i \in \samplingset} \|\noise[i]\|_1.
\vspace*{-2mm}
\end{equation}  
Combining \eqref{equ_upper_bound_complement_partition} with \eqref{equ_upper_bound_partial_partition} yields  
\begin{equation*} 
\hspace*{-10mm}\| \tilde{\vx}[\cdot] \|_{\edges}  \!=\!\| \tilde{\vx}[\cdot] \|_{\compbound}\!+\!\| \tilde{\vx}[\cdot] \|_{\partial \partition} 
 \stackrel{\eqref{equ_upper_bound_complement_partition}}{\leq}   \hspace*{-1mm} \frac{1}{\lambda} \sum_{i \in \samplingset} \|\noise[i]\|_1 \!+\! 2 \| \tilde{\vx}[\cdot] \|_{\partial \partition}
  \stackrel{\eqref{equ_upper_bound_partial_partition}}{\leq}   \hspace*{-1mm} ( \frac{1}{\lambda} \!+\!\frac{4\sqrt{\sigdim}/\lambda}{(L\!-\!\sqrt{\sigdim})}) \hspace*{-1mm} \sum_{i \in \samplingset}\hspace*{-1mm} \|\noise[i]\|_1.
\vspace*{-6mm}
\end{equation*} 
\vspace*{-3mm}
\end{proof} 

We highlight that the nLasso \eqref{equ_semi_sup_learning_problem} does not require the partition $\partition$ used for 
our signal model \eqref{equ_def_clustered_signal_model}. This partition is only used for the analysis of nLasso \eqref{equ_semi_sup_learning_problem}. 
Moreover, if the true underlying graph signal is of the form \eqref{equ_def_clustered_signal_model} and nLasso accurately learns this signal (Theorem \ref{lem_NSP1}), 
we can obtain the partition $\partition$ by thresholding the graph signal differences $\| \vx[i] \!-\! \vy[i] \|$ for $\edge{i}{j} \in \edges$ \citep{TrendGraph}. 

The bound \eqref{equ_bound_error_TV} characterizes the recovery error in terms of the semi-norm 
$\| \hat{\vx}[\cdot]-\vx[\cdot] \|_{\edges}$, and in general does not imply a small mean squared error. However, 
if $\| \hat{\vx}[\cdot]-\vx[\cdot] \|_{\edges}$ is small, we can identify the edges $\edge{i}{j}$ having 
large $\| \hat{\vx}[i]-\hat{\vx}[j] \|_2$ to obtain underlying clusters $\cluster_{l}$ (cf.\ \eqref{equ_def_clustered_signal_model}). 

Our second main contribution, beside Theorem \ref{lem_NSP1}, is to relate the NCC (cf.\ Definition \ref{def_NNSP}) to 
the network structure of the data graph $\graph$ via the existence of certain network flows \citep{KleinbergTardos2006}. 

Let us denote the neighborhood of node $i$ by $\mathcal{N}(i) \!\defeq\! \{ j \!\in\! \nodes: \edge{i}{j} \in \edges \}$ and $[p]\!\defeq\!\{1,2, \ldots p\}$.

\begin{definition} 
Consider a graph $\graph = (\nodes, \edges)$ with capacity matrix  $\mC \in \mathbb{R}_{+}^{\signalsize \times \signalsize}$. 
A flow with demands $\vd[i] \in \mathbb{R}^{\sigdim}$, for $i \in \nodes$, is a mapping $\flow[\cdot]: \nodes\times \nodes \rightarrow \mathbb{R}^{\sigdim}$ satisfying, for any $k \in [p]$, 
\vspace*{-1mm}
\begin{equation*} 
\hspace*{-6mm}\sum_{j \in \mathcal{N}(i)}\hspace*{-2mm} \flow_k(i,j) = d_k[i] \mbox{, for any }  i \!\in\! \nodes \mbox{, and }|\flow_k(i,j)| \leq C_{i,j} \mbox{ for any edge } \{i,j\} \!\in\!  {\edges}. 
\vspace*{-3mm}
\end{equation*} 
\end{definition} 

We can characterize the network topology by verifying the existence of certain network flows. In particular, the next 
results relates the existence of certain network flows with the NCC. 
\begin{lemma} 
\label{lem_NNSP_samplingset_suff_recovery}
Consider a dataset with data graph $\graph = (\nodes, \edges)$, whose nodes are partitioned into clusters 
$\partition$, capacity matrix  $\mC \in \mathbb{R}_{+}^{\signalsize \times \signalsize}$ with $C_{i,j} = W_{i,j}$ 
for all edges $\{i,j\} \!\in\!  \compbound$ and $C_{i,j} =L W_{i,j}$ for $\{i,j\} \!\in\!  \partial \partition$, and a sampling set $\samplingset$. 
If there exists, for any graph signal $\vz[\cdot]$ and any $k \in [\sigdim]$, a flow $\flow_{k}[\cdot]$ on 
${\graph}$ with $\flow_{k}(i,j)= {\rm sign} ( z_k[i] \!-\! z_k[j]) L \cdot W_{i,j} \mbox{ for } \{i,j\}\!\in\! \partial \partition$ 
and demands $|d_k[i]|  \leq K$ for every node $i\!\in\! \samplingset$ and $d_k[i] \!=\! 0$ for every node $i \!\in\!  \nodes \setminus \samplingset$, 
then $\samplingset$ satisfies the network compatibility condition with parameters $K,L >0$.
\end{lemma} 
\begin{proof} 
For a graph signal $\vz[\cdot]$ we denote, for each edge $\{i,j\} \in \partial \partition$, 
\begin{equation}
\label{equ_orientation_boundary}
b_k{(i,j)} \!=\! - b_k{(j,i)} \!=\! {\rm sign} ( z_k[i] \!-\! z_k[j]) \mbox{ for each } k \in [\sigdim]. 
\end{equation} 
Consider flows $\flow_k(i,j)$ on the data graph $\graph$ satisfying
\begin{align} 
\hspace{-5mm}
\label{equ_condition_flow_1}
\sum_{j \in \mathcal{N}(i)} \flow_k(i,j)   = 0  \mbox{ for all } i \notin \samplingset &\mbox{, } \hspace{0.5cm}
\bigg| \sum_{j \in \mathcal{N}(i)} \flow_k(i,j) \bigg|  \leq  K  \mbox{ for all } i \in \samplingset \\
|\flow_k(i,j)|  \leq W_{i,j}   \mbox{ for all } \{i,j\} \in \compbound  
&\mbox{, } \hspace{0.5cm} \flow_k(i,j) =  b_k{(i,j)} L W_{i,j}    \mbox{ for all } \{i,j\} \!\in\! \partial \partition. 
\label{equ_condition_flow_4}
\end{align} 
This yields, in turn, 
\begin{align}
L \| z[\cdot] \|_{\partial \partition} & = \sum_{\{i,j\}  \in \partial \partition} \|z[i] - z[j] \|_2 L W_{i,j} \leq \sum_{\{i,j\}  \in \partial \partition} \|z[i] - z[j] \|_1 L W_{i,j} \nonumber \\ 
& \hspace{-3mm} \stackrel{\eqref{equ_orientation_boundary},\eqref{equ_condition_flow_4}}{=} \sum_{k \in [\sigdim]} \sum_{\{i,j\}  \in \partial \partition} (z_k[i]\!-\!z_k[j])h_k(i,j). 
\label{equ_inequ_nnsp_single_cluster}
\end{align}
Since $\partial \partition = \edges \setminus \compbound$, developing  \eqref{equ_inequ_nnsp_single_cluster} yields
\begin{align}
L \| z[\cdot] \|_{\partial \partition} &\leq \sum_{k \in [\sigdim]} \bigg( \sum_{\{i,j\}  \in \edges} (z_k[i]\!-\!z_k[j])h_k(i,j) - \sum_{\{i,j\}  \in \compbound} (z_k[i]\!-\!z_k[j])h_k(i,j) \bigg) \nonumber \\
& \hspace*{-0mm} \stackrel{}{=} \sum_{i  \in \nodes} \sum_{k \in [\sigdim]} z_k[i] \sum_{j \in \mathcal{N}(i)} \flow_k(i,j) -  \sum_{\{i,j\}  \in \compbound} \sum_{k \in [\sigdim]} (z_k[i]\!-\!z_k[i])h_k(i,j).
\label{equ_inequ_nnsp_nodes_decomposed}
\end{align}
Applying \eqref{equ_condition_flow_1},\eqref{equ_condition_flow_4} into \eqref{equ_inequ_nnsp_nodes_decomposed} yields further
\begin{align*}
L \| z[\cdot] \|_{\partial \partition}  \!\stackrel{\eqref{equ_condition_flow_1},\eqref{equ_condition_flow_4}}{\leq}\!  \!K\!  \sum_{i  \in \samplingset}  \|z[i]\|_1 \!+\hspace*{-3mm}\sum_{\{i,j\}  \in \compbound} \hspace*{-3mm}W_{ \{i,j\} } \| z[i]\!-\!z[j]\|_1   \leq  \sqrt{p} (K  \sum_{i  \in \samplingset}   \|z[i]\|_2 +    \| z[\cdot] \|_{\compbound} ).
\vspace*{-3mm}
\end{align*}
Thus, the condition \eqref{equ_ineq_multcompcondition_condition} is verified.
\end{proof} 


\bibliographystyle{plainnat}
\bibliography{SLPBib}

\end{document}